\newif\iflong
\longtrue

\documentclass[12pt,letter]{article}
\usepackage[papersize={8.5in,11in},margin=1in]{geometry}
\usepackage{enumitem}
\usepackage{soul}
\usepackage{color}
\usepackage{graphicx,subfigure,amsmath,amsthm,amssymb,amsfonts,bm,epsfig,epsf,url,dsfont,pifont,xspace}
\usepackage{times}
\usepackage{bbm}
\usepackage{booktabs}
\usepackage{cases}
\usepackage[small,bf]{caption}
\usepackage{algorithm,algorithmicx,algpseudocode} 

\usepackage{natbib}

\newcommand{\dtlinkcolor}{{0.8 0.8 1}} 
\usepackage[hypertexnames=false,hyperindex=true,pdfpagemode=UseOutlines,bookmarksnumbered=true,bookmarksopen=true,bookmarksopenlevel=2,pdfstartview=FitH,pdfborder={0 0 1},linkbordercolor=\dtlinkcolor,citebordercolor=\dtlinkcolor,urlbordercolor=\dtlinkcolor,pagebordercolor=\dtlinkcolor]{hyperref}
\hypersetup{pageanchor=false,pdfpagelabels}
\usepackage[capitalize,nameinlink]{cleveref}

\renewcommand{\phi}{\varphi}

\newcommand{\E}{\mathbb{E}}

\newcommand{\R}{\mathbb{R}}

\def\ds1{\mathds{1}}
\renewcommand{\epsilon}{\varepsilon}

\renewcommand{\tilde}{\widetilde}

\newlength{\minipagewidth}
\newcommand{\beq}{\begin{equation}}
\newcommand{\eeq}{\end{equation}}

\newcommand{\beqa}{\begin{eqnarray}}
\newcommand{\eeqa}{\end{eqnarray}}

\newcommand{\beqan}{\begin{eqnarray*}}
\newcommand{\eeqan}{\end{eqnarray*}}

\def\ba#1\ea{\begin{align*}#1\end{align*}} 
\def\banum#1\eanum{\begin{align}#1\end{align}} 

\newcommand{\defeq}{\stackrel{\mathrm{\scriptscriptstyle def}}{=}}

\newcommand{\Trun}{\mathcal{T}}
\newcommand{\bell}{\bar{\ell}}
\newcommand{\MYGA}{{\hyperref[alg:MYGA]{\mathtt{MYGA}}}}

\newtheorem{theorem}{Theorem}[section]
\newtheorem{lemma}[theorem]{Lemma}
\newtheorem{claim}[theorem]{Claim}

\newtheorem{definition}[theorem]{Definition}

\newtheorem*{rep@theorem}{\rep@title}
\newcommand{\newreptheorem}[2]{%
\newenvironment{rep#1}[1]{%
 \def\rep@title{#2 \ref{##1}}%
 \begin{rep@theorem}}%
 {\end{rep@theorem}}}

\makeatother

\newreptheorem{theorem}{Theorem}
\newreptheorem{lemma}{Lemma}
\newreptheorem{proposition}{Proposition}
\newreptheorem{claim}{Claim}
\newreptheorem{fact}{Fact}

\theoremstyle{definition}

\newtheorem{example}[theorem]{Example}

\numberwithin{equation}{section}

\newcommand{\defem}[1]{\emph{\textsf{\small #1}}}
\newcommand{\da}{\text{\ding{172}}\xspace}
\newcommand{\db}{\text{\ding{173}}\xspace}
\newcommand{\dc}{\text{\ding{174}}\xspace}
\newcommand{\dd}{\text{\ding{175}}\xspace}

\newcommand{\rref}[1]{\mbox{\cref{#1}}}
\newcommand{\namedref}[2]{\mbox{\hyperref[#2]{#1~\ref*{#2}}}}

\newcommand{\figurerefb}[2]{\mbox{\hyperref[#1]{Figure~\ref*{#1}#2}}}

\newcommand{\lemmaref}[1]{\namedref{Lemma}{#1}}
\newcommand{\claimref}[1]{\namedref{Claim}{#1}}

\newcommand{\algorithmref}[1]{\namedref{Algorithm}{#1}}

\newcommand{\equationref}[1]{\mbox{\hyperref[#1]{(\ref*{#1})}}}
\renewcommand{\eqref}{\equationref}

\newcommand{\lineref}[1]{\namedref{Line}{#1}}
\newcommand{\problemref}[1]{\mbox{\hyperref[#1]{Problem (\ref*{#1})}}}

\iflong

\theoremstyle{plain} 
\setitemize{itemsep=0mm, topsep=2mm, leftmargin=8mm}
\setenumerate{itemsep=0mm, topsep=2mm, leftmargin=8mm}
\newcommand{\parhead}[1]{\smallskip \noindent {\bfseries\boldmath\ignorespaces #1.}\hskip 0.9em plus 0.3em minus 0.3em}

\AtBeginDocument{%
 \abovedisplayskip=6pt minus 1pt
 \abovedisplayshortskip=4pt plus 1pt
 \belowdisplayskip=6pt minus 1pt
 \belowdisplayshortskip=4pt plus 1pt
}

\else

\newcommand{\parhead}[1]{\smallskip\noindent {\bfseries\boldmath\ignorespaces #1.}\hskip 0.9em plus 0.3em minus 0.3em}

\allowdisplaybreaks
\AtBeginDocument{%
 \abovedisplayskip=4pt minus 1pt
 \abovedisplayshortskip=2pt plus 1pt
 \belowdisplayskip=4pt minus 1pt
 \belowdisplayshortskip=2pt plus 1pt
}
\setitemize{itemsep=0mm, leftmargin=5mm, topsep=0mm}
\setenumerate{itemsep=0mm, leftmargin=5mm, topsep=0mm}
\usepackage[noindentafter]{titlesec}
\titlespacing\section{0pt}{4pt plus 0pt minus 1pt}{2pt plus 0pt minus 1pt}
\titlespacing\subsection{0pt}{4pt plus 0pt minus 1pt}{1pt plus 1pt minus 1pt}
\titlespacing\subsubsection{0pt}{4pt plus 0pt minus 1pt}{1pt plus 1pt minus 1pt}

\newtheoremstyle{slplain}
  {.4\baselineskip\@plus.1\baselineskip\@minus.1\baselineskip}
  {.3\baselineskip\@plus.1\baselineskip\@minus.1\baselineskip}
  {\itshape}
  {}
  {\bfseries}
  {.\xspace}
  { }
  {}
\theoremstyle{slplain} 

\fi

\begin{document}

\title{Make the Minority Great Again: \\
First-Order Regret Bound for Contextual Bandits}

\newcommand{\authorname}[1]{\makebox[4.1cm][c]{#1}}
\author{
\authorname{Zeyuan Allen-Zhu} \\
\footnotesize{\textsf{\href{mailto:zeyuan@csail.mit.edu}{\color{black}zeyuan@csail.mit.edu}}} \\
Microsoft Research
\and
\authorname{S\'ebastien Bubeck} \\
\footnotesize{\textsf{\href{mailto:sebubeck@microsoft.com}{\color{black}sebubeck@microsoft.com}}} \\
Microsoft Research
\and
\authorname{Yuanzhi Li\thanks{This work was done while Y. Li was an intern at Microsoft Research.}} \\
\footnotesize{\textsf{\href{mailto:yuanzhil@cs.princeton.edu}{\color{black}yuanzhil@cs.princeton.edu}}} \\
Princeton University
}

\date{\today}

\maketitle

\begin{abstract}
Regret bounds in online learning compare the player's performance to $L^*$, the optimal performance in hindsight with a fixed strategy. Typically such bounds scale with the square root of the time horizon $T$. The more refined concept of \emph{first-order} regret bound replaces this with a scaling $\sqrt{L^*}$, which may be much smaller than $\sqrt{T}$.
It is well known that minor variants of standard algorithms satisfy first-order regret bounds in the full information and multi-armed bandit settings. In a COLT 2017 open problem~\cite{AKLLS2017}, Agarwal, Krishnamurthy, Langford, Luo, and Schapire raised the issue that existing techniques do not seem sufficient to obtain first-order regret bounds for the contextual bandit problem. In the present paper, we resolve this open problem by presenting a new strategy based on augmenting the policy space.
\end{abstract}

\section{Introduction}
The contextual bandit problem is an influential extension of the classical multi-armed bandit. It can be described as follows. Let $K$ be the number of actions, $E$ a set of experts (or ``policies"), $T$ the time horizon, and denote $\Delta_K=\{x \in [0,1]^K : \sum_{i=1}^K x(i) =1\}$. At each time step $t=1, \hdots, T$,
\begin{itemize}
\item The player receives from each expert $e \in E$ an ``advice" $\xi_t^e \in \Delta_K$.
\item Using advices and previous feedbacks, the player selects a probability distribution $p_t \in \Delta_K$.
\item The adversary selects a loss function $\ell_t : [K] \rightarrow [0,1]$.
\item The player plays an action $a_t \in [K]$ at random from $p_t$ (and independently of the past).
\item The player's suffered loss is $\ell_t(a_t) \in [0,1]$, which is also the only feedback the player receives about the loss function $\ell_t$.
\end{itemize}
The player's performance at the end of the $T$ rounds is measured through the regret with respect to the best expert:
\begin{equation} \label{eq:regret}
R_T \defeq
\max_{e \in E} \Big\{ \E \Big[ \sum_{t=1}^T \ell_t(a_t) - \langle \xi_t^e, \ell_t \rangle \Big] \Big\}
=
\max_{e \in E} \Big\{ \E \Big[ \sum_{t=1}^T \langle p_t - \xi_t^e, \ell_t \rangle \Big] \Big\}
\enspace.
\end{equation}

A landmark result by \cite{ACFS03} is that a regret of order $O(\sqrt{T K \log(|E|)})$ is achievable in this setting. The general intuition captured by regret bounds is that the player's performance is equal to the best expert's performance up to a term of lower order. However the aforementioned bound might fail to capture this intuition if $T \gg L^*_T \defeq \min_{e \in E} \E \sum_{t=1}^T \langle \xi_t^e , \ell_t \rangle$. It is thus natural to ask whether one could obtain a stronger guarantee where $T$ is essentially replaced by $L^*_T$. This question was posed as a COLT 2017 open problem~\cite{AKLLS2017}. Such bounds are called first-order regret bounds, and they are known to be possible with full information \cite{ACFS03}, as well as in the multi-armed bandit setting \cite{AAGO06} (see also \cite{FLLST2016} for a different proof) and the semi-bandit framework  \cite{Neu2015, LST2017}. Our main contribution is a new algorithm for contextual bandit, which we call $\MYGA$ (see \rref{sec:alg}), and for which we prove the following first-order regret bound, thus resolving the open problem.
\begin{theorem} \label{th:main}
For any loss sequence such that $\min_{e \in E} \E \sum_{t=1}^T \langle \xi_t^e , \ell_t \rangle \leq L^*$ one has that $\MYGA$ with $\gamma = \Theta(\eta)$ and $\eta = \Theta \Big(\min \big \{\frac{1}{K}, \sqrt{\frac{\log(|E| T)}{K L^*}} \big\} \Big)$ satisfies
\[
R_T \leq O\left(\sqrt{K \log(|E| T) L^*} + K \log(|E| T) \right) \,.
\]
\end{theorem}

\section{Algorithm Description} \label{sec:alg}

\newcommand{\Imaj}{I^{\mathsf{maj}}}
\newcommand{\Imino}{I^{\mathsf{mino}}}
In this section we describe the $\MYGA$ algorithm.

\subsection{Truncation}
We introduce a truncation operator $\Trun^k_s$ that takes as input an index $k\in [K]$ and a threshold $s\in [0,\frac{1}{2}]$. Then, treating the first $k$ arms as ``majority arms'' and the last $K-k$ arms as ``minority arms,'' $\Trun^k_{s}$ redistributes ``multiplicatively" the probability mass of all minority arms below threshold $s$ to the majority arms.
\begin{definition}\label{def:truncKdef}
For $k\in [K]$ and $s\in(0,\frac{1}{2}]$, the \defem{truncation operator} $\Trun_{s}^k \colon \Delta_K \to \Delta_K$ is defined as follows. Given any $q\in \Delta_K$, then we set
\begin{equation*}
\Trun^k_s q(i) = \left\{
                 \begin{array}{ll}
                   0, & \hbox{$i > k$ and $q(i)\leq s$;} \\
                   q(i), & \hbox{$i > k$ and $q(i) > s$;} \\
                   q(i) \cdot \big( 1 + \frac{\sum_{j:j>k \,\wedge\, q(j)\leq s} q(j)}{\sum_{j\leq k} q(j)} \big)
                    , & \hbox{$i \leq k$.}
                 \end{array}
               \right.
\end{equation*}
\end{definition}
Equivalently one can define $\Trun^k_s q(i)$ for the majority arms $i\leq k$ with the following implicit formula:
\begin{equation} \label{eq:implicittruncation}
\Trun^k_s q(i) = \frac{q(i)}{\sum_{j \leq k} q(j)} \sum_{j \leq k} \Trun^k_s q(j) \,.
\end{equation}
To see this it suffices to note that the amount of mass in the majority arms is given by 
\[
\sum_{j \leq k} \Trun^k_s q(j) = 1 - \sum_{j > k} \Trun^k_s q(j) = 1 - \sum_{j:j>k \,\wedge\, q(j) > s} q(j) = \sum_{j\leq k} q(j) + \sum_{j:j>k \,\wedge\, q(j)\leq s} q(j) \,.
\]

\begin{example}
If $K=2$, then $\Trun_s^1 q$ simply adds $q(2)$ into $q(1)$ if $q(2) \leq s$.
\end{example}
\begin{example}
An example with $K=11$ and $k=3$ is as follows:
\begin{align*}
q &= \big(
&0.2 && 0.1 && 0.2 && 0.1 && 0.1 && 0.1 && 0.05 && 0.05 && 0.04 && 0.03 && 0.03 \big)
\\
\Trun^3_{0.02} q &=\big(
&0.2 && 0.1 && 0.2 && 0.1 && 0.1 && 0.1 && 0.05 && 0.05 && 0.04 && 0.03 && 0.03 \big)
\\
\Trun^3_{0.03} q &=\big(
&0.224 && 0.112 && 0.224 && 0.1 && 0.1 && 0.1 && 0.05 && 0.05 && 0.04 && 0 && 0 \big)
\\
\Trun^3_{0.04} q &=\big(
&0.24 && 0.12 && 0.24 && 0.1 && 0.1 && 0.1 && 0.05 && 0.05 && 0 && 0 && 0 \big)
\\
\Trun^3_{0.05} q &=\big(
&0.28 && 0.14 && 0.28 && 0.1 && 0.1 && 0.1 && 0 && 0 && 0 && 0 && 0 \big)
\\
\Trun^3_{0.1} q &=\big(
&0.4 && 0.2 && 0.4 && 0 && 0 && 0 && 0 && 0 && 0 && 0 && 0 \big)
\\
\Trun^3_{0.2} q &=\big(
&0.4 && 0.2 && 0.4 && 0 && 0 && 0 && 0 && 0 && 0 && 0 && 0 \big)
\\
\Trun^3_{0.5} q &=\big(
&0.4 && 0.2 && 0.4 && 0 && 0 && 0 && 0 && 0 && 0 && 0 && 0 \big)
\end{align*}
 \end{example}

\subsection{Informal description}
$\MYGA$ is parameterized by two parameters: a classical learning rate $\eta>0$, and a thresholding parameter $\gamma \in \frac{1}{2T} \mathbb{N} = \{ \frac{1}{2T}, \frac{2}{2T}, \frac{3}{2T} ,\dots \} $. Also let $S = (\gamma, 1/2] \cap \frac{1}{2T} \mathbb{N} = (\gamma, 1/2] \cap \{ \frac{1}{2T}, \frac{2}{2T}, \frac{3}{2T} ,\dots \}$

At a high level, a key feature of $\MYGA$ is to introduce a set of \emph{auxiliary} experts, one for each $s \in S$. More precisely, in each round $t$, after receiving expert advices $\{\xi^e_t\}_{e\in E}$, $\MYGA$ calculates a distribution $\xi^s_t \in \Delta_K$ for each $s\in S$. Then, $\MYGA$ uses the standard exponential weight updates on $E' = E \cup S$ with learning rate $\eta>0$, to calculate a weight function $w_t \in \mathbb{R}_+^{|E|+|S|}$ ---see \eqref{eq:expweights}. Then, it computes
\begin{itemize}
\item $\zeta_t \in \Delta_K$, the weighted average of expert advices in $E$: $\zeta_t = \frac{1}{\sum_{e\in E} w_t(e)} \sum_{e \in E} w_t(e) \cdot \xi_t^e$.
\item $q_t \in \Delta_K$, the weighted average of expert advices in $E'$: $q_t = \frac{1}{\|w_t\|_1} \sum_{e \in E'} w_t(e) \cdot \xi_t^e$.
\end{itemize}
Using these information, $\MYGA$ calculates the probability distribution $p_t \in \Delta_K$ from which the arm is played at round $t$.

Let us now explain how $p_t$ and $\xi^s_t$, $s \in S$ are defined. First we remark that in the contextual bandit setting, the arm index has no real meaning since in each round $t$ we can permute the arms by some $\pi_t \colon [K] \to [K]$ and permute the expert's advices and the loss vector by the same $\pi_t$. For this reason, throughout this paper, we shall assume
\[
\zeta_t(1) \geq \zeta_t(2 ) \geq \cdots \zeta_t(K) \,.
\]
Let us define the ``pivot" index $k_t = \min\{ i \in [K] : \sum_{j \leq i} \zeta_t(j) \geq 1/2\}$. Then, in order to perform truncation, $\MYGA$ views the first $k_t$ arms as ``majority arms'' and the last $K-k_t$ arms as ``minority arms'' of the current round $t$. At a high level we will have:
\begin{itemize}
\item the distribution to play from is $p_t = \Trun_\gamma^{k_t} q_t$.
\item Each auxiliary expert $s\in S$ is defined by $\xi_t^s = \Trun_s^{k_t} q_t$.
\end{itemize}
We now give a more precise description in \algorithmref{alg:MYGA}.

\begin{algorithm*}[t!]
\caption{$\MYGA$ (Make the minoritY Great Again) \label{alg:MYGA}}
\begin{algorithmic}[1]
\Require learning rate $\eta >0$, threshold parameter $\gamma \in \frac{1}{2T} \mathbb{N}$
\vspace{2pt}
\State $S \gets (\gamma, 1/2] \cap \frac{1}{2T} \mathbb{N}$ and $w_1 \gets (1,\ldots,1) \in \R^{|E| + |S|}$

\For {$t=1$ \textbf{to} $T$}
\State 
receive advices $\xi_t^e\in \Delta_K$ from each expert $e\in E$

\State weighted average $\zeta_t \gets \frac{\sum_{e\in E}w_t(e) \xi^e_t}{\sum_{e\in E} w_t(e)} \in \Delta_K$

\State
assume $\zeta_t(1) \geq \zeta_t(2 ) \geq \cdots \zeta_t(K)$ wlog. by permuting the arms

\State
$k_t \gets \min\{ i \in [K] : \sum_{j \leq i} \zeta_t(j) \geq 1/2\}$
\Comment{the first $k_t$ arms are majority arms}

\State
find $q_{t} \in \Delta_K$ such that
\Comment{$q_t$ can be found in time $O(K |S|) = O(K T)$, see \rref{lem:itexists}}
\begin{equation} \label{eq:implicitdef}
\textstyle q_{t} = \frac{1}{\sum_{e \in E} w_t(e) + \sum_{s \in S} w_t(s)} \Big(\sum_{e \in E} w_t(e) \xi_t^e + \sum_{s \in S} w_t(s) \Trun^{k_t}_s q_t \Big) ~.
\end{equation}

\State $\xi_t^s \gets \Trun^{k_t}_{s} q_t$ for every $s \in S$ \quad and \quad $p_t \gets \Trun^{k_t}_{\gamma} q_t$

\State draw an arm $a_t \in [K]$ from probability distribution $p_t$ and receive feedback $\ell_t(a_t)$

\vspace{2pt}
\State compute loss estimator $\tilde{\ell}_t \in \mathbb{R}_{+}^K$ as  $\tilde{\ell}_t(i) = \frac{\ell_{t}(i)}{p_{t}(i)} \ds1_{i = a_t}$
\vspace{2pt}

\State update the exponential weights for any $e \in E \cup S$:
\begin{equation} \label{eq:expweights}
\textstyle w_{t+1}(e) = \exp \Big(  - \eta \sum_{r=1}^t \langle \xi_r^e, \tilde{\ell}_r \rangle \Big) ~.
\end{equation}

\EndFor
\end{algorithmic}
\end{algorithm*}

\section{Preliminaries}
\begin{definition}\label{def:truncLoss}
For analysis purpose, let us define the \defem{truncated loss} $\bell_t (i) \defeq \ell_t(i) \ds1\{p_t(i) > 0\}$, so that
$$\E_{a_t} \big[ \langle \tilde{\ell}_t, p_t \rangle \big] = \langle \bell_t, p_t \rangle = \langle \ell_t, p_t \rangle $$
\end{definition}

We next derive two lemmas that will prove useful to isolate the properties of the truncation operator $\Trun_s^k$ that are needed to obtain a first-order regret bound.

\begin{lemma} \label{lem:1}
Let $\gamma \in [0,1]$ and assume that for all $i \in [K]$, $(1- c K \gamma) p_t(i) \leq q_t(i)$ for some universal constant $c>0$, and that $p_t(i) \neq 0 \Rightarrow p_t(i) \geq q_t(i)$. Then one has
\begin{equation} \label{eq:stdregret}
(1- c K \gamma) L_T - L_T^* \leq \frac{\log(|E'|)}{\eta} + \frac{\eta}{2} \E \sum_{t=1}^T \|\bell_t\|_2^2 ~.
\end{equation}
\end{lemma}

\begin{proof}
Using $\langle p_t , \ell_t \rangle = \langle p_t , \bell_t \rangle$, $\langle - \xi_t^e, \ell_t \rangle \leq \langle - \xi_t^e, \bell_t \rangle$, and $(1- c K \gamma) p_t(i) \leq q_t(i)$, we have
\[
(1- c K \gamma) L_T - L_T^* \leq \max_{e \in E'} \E \sum_{t=1}^T \langle (1- c K \gamma) p_t - \xi_t^e, \bell_t \rangle \leq \max_{e \in E'} \E \sum_{t=1}^T \langle q_t - \xi_t^e, \bell_t \rangle \,.
\]
The rest of the proof follows from standard argument to bound the regret of Exp4, see e.g., [Theorem 4.2, \cite{BC12}] (with the minor modification that the assumption on $p_t$ implies that $\tilde{\ell}_t(i) \leq \frac{\ell_t(i)}{q_t(i)} \ds1\{i = a_t\}$).
\end{proof}

The next lemma is straightforward.

\begin{lemma} \label{lem:2}
In addition to the assumptions in \rref{lem:1}, assume that there exists some numerical constants $c', c'' \geq 0$ such that
\begin{equation} \label{eq:toshow}
\gamma \ \E \sum_{t=1}^T \|\bell_t\|_2^2 \leq 2 \ c' \ (\eta+\gamma) \ K \ L_T +2 \ c'' \ \frac{\log(|E'|)}{\eta} ~.
\end{equation}
Then one has
\[
\left(1- c K \gamma - \left(\eta + \frac{\eta^2}{\gamma}\right) c' K)\right) (L_T - L_T^*) \leq \left( \frac{1}{\eta} + \frac{c''}{\gamma} \right) \log(|E'|) + \left(c K \gamma + \left(\eta + \frac{\eta^2}{\gamma}\right) c' K \right) L_T^* ~.
\]
\end{lemma}

We now see that it suffices to show that $\MYGA$ satisfies the assumptions of \rref{lem:1} and \rref{lem:2} for $\gamma \simeq \eta$, and $\eta \simeq \min\left\{ \frac{1}{K}, \sqrt{\frac{\log(|E'|)}{K L_T^*}}\right \}$ (assume that $L_T^*$ is known), in which case one obtains a bound of order $\sqrt{K \log(|E'|) L_T^*} + K \log(|E'|)$.
\newline

In fact the assumption of \rref{lem:1} will be easily verified, and the real difficulty will be to prove \eqref{eq:toshow}.
We observe that the standard trick of thresholding the arms with probability below $\gamma$ would yield \eqref{eq:toshow} with the right hand side replaced by $L_T$, and in turn this leads to a regret of order $(L_T^*)^{2/3}$. Our goal is to improve over this naive argument.

\section{Proof of the $2$-Armed Case}
The goal of this section is to explain how our $\MYGA$ algorithm arises naturally. To focus on the main ideas we restrict to the case $K=2$. The complete formal proof of \rref{th:main} is given in \rref{sec:proof}.

Recall we have assumed without loss of generality that $\zeta_t(1) \geq \zeta_t(2)$ for each round $t\in [T]$. This implies $k_t = 1$ because $\zeta_t(1) \geq \frac{1}{2}$.
In this simple case, for $s \in [0,1/2]$, we abbreviate our truncation operator $\Trun_s^{k_t}$ as $\Trun_s$, and it acts as follows. Given $q \in \Delta_2$
$$ \text{if $q(2) \leq s$ we have $\Trun_s q = (1,0)$; \quad and if $q(2)>s$ we have $\Trun_s q = q$.} $$
In particular, we have $q_t(1) \geq q_t(2)$ and $p_t(1) \geq p_t(2)$ for all $t \in [T]$. We refer to arm $1$ as the majority arm and arm $2$ as the minority arm.
We denote $M = \E \sum_{t=1}^T \bell_t(1)$ as the loss of the majority arm and $m = \E \sum_{t=1}^T \bell_t(2)$ as the loss of the minority arm.

Since $\ell_t \in [0,1]^K$ and $K=2$, we have
\begin{equation} \label{eq:omwtoshow}
\textstyle \E \sum_{t=1}^T \|\bell_t\|_2^2 \leq \E \sum_{t=1}^T \bell_t(1) + \bell_t(2) = M + m ~.\end{equation}
Observe also that one always has $L_T \geq \frac{1}{2} M$ (indeed $p_t(1) \geq q_t(1) \geq 1/2$), and thus the whole game to prove \eqref{eq:toshow} is to upper bound the minority's loss $m$.

\subsection{When the minority suffers small loss}
Assume that $m \leq (c'-1) M$ for some constant $c'>0$. Then, because $M \leq 2 L_T$, one can directly obtain \eqref{eq:toshow} from \eqref{eq:omwtoshow} with $c''=0$. In words, when the minority arm has a total loss comparable to the majority arm, simply playing from $\zeta_t$ would satisfy a first-order regret bound.

Our main idea is to somehow enforce this relation $m\lesssim M$ between the minority and majority losses, by ``truncating'' probabilities appropriately. Indeed, recall that if after some truncation we have $p_t(2) = 0$, then it satisfies $\bell_t(2) = 0$ so the minority loss $m$ can be improved.

\subsection{Make the minority great again}
Our key new insight is captured by the following lemma which is proved using an  integral averaging argument.

\begin{definition}
For each $s\geq \gamma$, let $L_t^s \defeq \E \sum_{t=1}^T \langle \Trun_s q_t, \ell_t \rangle$ be the expected loss if the truncated strategy $\Trun_s q_t \in \Delta_K$ is played at each round. \end{definition}
\begin{lemma} \label{lem:key}
As long as $m - M > 0$,
\[
\exists s \in (\gamma,1/2] \colon \quad m-M \leq \frac{L_T - L_T^s}{\gamma} ~.
\]
\end{lemma}
In words, if $m$ is large, then it must be that was a much better threshold $s$ compared to $\gamma$, that is $L_T - L_T^s$ is large.

\begin{proof}[Proof of \rref{lem:key}]
For any $s \geq \gamma$, define the function
$$\textstyle f(s) \defeq \E \sum_{t=1}^T \ds1\{q_t(2) \leq s\} (\bell_t(1) - \bell_t(2)) \enspace.$$
Let us pick $s \in [\gamma,1/2]$ to minimize $f(s)$, and breaking ties by choosing the smaller value of $s$.
We make several observations:
\begin{itemize}
\item $f(\gamma) \geq 0$ because for any $t$ with $q_t(2) \leq \gamma$ we must have $\bell_t(2)=0$.
\item $f(1/2) = M - m < 0$.
\item $s > \gamma$ because $f(s) \leq f(1/2) < 0$.
\end{itemize}
Let us define the points
\[
\text{$s_0 \defeq \gamma$ \quad and \quad $\{s_1 < \hdots < s_m\} \defeq (\gamma, s] \cap \{q_1(2), \hdots, q_T(2)\}$.}
\]
Note that the tie-breaking rule for the choice of $s$ ensures $s_m = s$ (if $s_m<s$ then it must satisfy $f(s_m) = f(s)$ giving a contradiction).

Using the identity
\begin{equation} \label{eq:wodisc}
\sum_{t=1}^T \langle \Trun_s q_t, \bell_t \rangle = \sum_{t=1}^T \langle q_t, \bell_t \rangle + \ds1\{q_t(2) \leq s\} q_t(2) (\bell_t(1) - \bell_t(2)) \enspace,
\end{equation}
we calculate that
\begin{align*}
L_T - L_T^s
&= \E \sum_{t=1}^T \langle \Trun_\gamma q_t - \Trun_s q_t, \ell_t \rangle
= \E \sum_{t=1}^T \langle \Trun_\gamma q_t - \Trun_s q_t, \bell_t \rangle \\
& = \E \sum_{t=1}^T (\ds1\{q_t(2) \leq \gamma\} - \ds1\{q_t(2) \leq s\}) q_t(2) (\bell_t(1) - \bell_t(2))  \\
& = \E \sum_{t=1}^T \sum_{i=1}^m - s_i \ds1\{q_t(2) = s_i\} (\bell_t(1) - \bell_t(2)) \\
&= \sum_{i=1}^m s_i (f(s_{i-1}) - f(s_{i})) = \sum_{i=1}^{m-1} (s_{i+1} - s_{i}) f(s_i) + s_1 f(s_0) - s_m f(s_m) ~.
\end{align*}
Since $f(s_0)\geq 0$, $f(s_i) \geq f(s)$ and $s = s_m$, we conclude that
\begin{equation*}
L_T - L_T^s \geq (s_m - s_1) f(s_m) - s_m f(s_m) = - s_1 f(s_m) \geq \gamma (m - M) ~. \qedhere
\end{equation*}
\end{proof}

Given \rref{lem:key}, a very intuitive strategy start to emerge. Suppose we can somehow get an upper bound of the form
\begin{equation} \label{eq:toproveX}
\textstyle L_T - L_T^s \leq O\big( \frac{\log(|E'|)}{\eta} + \eta (m+M) + \gamma L_T \big) ~.
\end{equation}
Then, putting this into \rref{lem:key} and using $M \leq 2 L_T$,  we have for any $\gamma \geq 2 \eta$,
$$
\textstyle \gamma m \leq O\big( \frac{\log(|E'|)}{\eta} + \gamma L_T \big) ~.
$$
In words, the minority arm also suffers from a small loss (and thus is great again!) Putting this into \eqref{eq:omwtoshow}, we immediately get \eqref{eq:toshow} as desired and finish the proof of \rref{th:main} in the case $K=2$.

Thus, we are left with showing \eqref{eq:toproveX}. The main idea is to add the truncated strategy $\Trun_s q_t$ as an additional auxiliary expert. If we can achieve this, then \eqref{eq:toproveX} can be obtained from the regret formula in \rref{lem:1}.

\subsection{Expanding the set of experts}
Assume for a moment that we somehow expand the set of experts into $E' \supset E$ so that:
\begin{equation} \label{eq:key}
\forall s \in (\gamma,1/2], \exists e \in E' \text{ such that for all } t \in [T], \xi_t^{e} = \Trun_s q_t ~.
\end{equation}
Then clearly \eqref{eq:toproveX} would be satisfied using \rref{lem:1}, \eqref{eq:omwtoshow} and $L_T^* \leq L_T^s$ (the loss of an expert should be no better than the loss of the best expert $L_T^*$).

There are two issues with condition \eqref{eq:key}: first, it self-referential, in the sense that it assumes $\{\xi_t^e\}_{e\in E'}$ satisfies a certain form depending on $q_t$ while $q_t$ is defined via $\{\xi_t^e\}_{e\in E'}$ (recall \eqref{eq:implicitdef}); and second, it potentially requires to have an infinite number of experts (one for each $s \in (\gamma, 1/2]$).

Let us first deal with the second issue via discretization.
\begin{lemma} \label{lem:key2}
In the same setting as \rref{lem:key}, there exists $s \in S \defeq (\gamma, 1/2] \cap \frac{1}{2T} \mathbb{N}$ such that
\[
m-M \leq \frac{1+L_T - L_T^s}{\gamma} ~.
\]
\end{lemma}
\begin{proof}
For $x \in \R$ let $\underline{x}$ be the smallest element in $[x,+\infty) \cap \frac{1}{2T} \mathbb{N}$. For any $s \in S$ we can rewrite \eqref{eq:wodisc} as (note that $x \leq s \Leftrightarrow \bar{x} \leq s$)
\[
\langle \Trun_s q_t, \bell_t \rangle = \langle q_t, \bell_t \rangle + \ds1\{\underline{q_t(2)} \leq s\} \underline{q_t(2)} (\bell_t(1) - \bell_t(2)) +\epsilon_{t,s} ~,
\]
where $|\epsilon_{t,s}| \leq 1/2T$. Using the same proof of \rref{lem:key}, and redefining
$$\textstyle f(s) \defeq \E \sum_{t=1}^T \ds1\{\underline{q_t(2)} \leq s\} (\bell_t(1) - \bell_t(2)) \enspace.$$
we get that there exists $s_1, \hdots, s_m \in S \defeq (\gamma, \frac{1}{2}] \cap \frac{1}{2T} \mathbb{N}$ and $\epsilon \in [-1, 1]$ such that
\[
L_T - L_T^s = \epsilon + \sum_{i=1}^m s_i (f(s_{i-1}) - f(s_{i})) ~.
\]
The rest of the proof now follows from the same proof of \rref{lem:key}, except that we minimize $f(s)$ over $s \in S$ instead of $s\in[\gamma, \frac{1}{2}]$.
\end{proof}

Thus, instead of \eqref{eq:key}, we only need to require
\begin{equation} \label{eq:key1}
\forall s \in S, \exists e \in E' \text{ such that for all } t \in [T], \xi_t^{e} = \Trun_s q_t ~.
\end{equation}

We now resolve the self-referentiality of \eqref{eq:key1} by defining simultaneously $q_t$ and $\xi_t^e, e \in S$ as follows.
Consider the map $F_t : [0,1/2] \rightarrow [0,1/2]$ defined by:
$$F_t(x) = \frac{1}{\sum_{e \in E} w_t(e) + \sum_{s \in S} w_t(s)} \left(\sum_{e \in E} w_t(e) \xi_t^e(2) + \sum_{s \in S} w_t(s) x \ds1\{x > s\} \right) ~.$$
It suffices to find a fixed point $x = F_t(x)$: indeed, setting
$$\text{$q_t \defeq (1-x, x)$ \quad and \quad $\xi_t^s(2) \defeq x \ds1\{x > s\} = \Trun_s q_t$ for $s \in S$,}$$
we have both \eqref{eq:key1} holds and $q_t = \frac{1}{\|w_t\|_1} \sum_{e \in E'} w_t(e) \cdot \xi_t^e$ is the correct weighted average of expert advices in $E' = E \cup S$

Finally, $F_t$ has a fixed point since it is a nondecreasing function from a closed interval to itself. It is also not hard to find such a point algorithmically.

This concludes the (slightly informal) proof for $K=2$. We give the complete proof for arbitrary $K$ in the next section.

\section{Proof of Theorem \ref{th:main}} \label{sec:proof}

In this section, we assume $q_t \in \Delta_K$ satisfies \eqref{eq:implicitdef} and we defer the constructive proof of finding $q_t$ to \rref{sec:qt}. 
Recall the arm index has no real meaning so without loss of generality we have permuted the arms so that 
$$\zeta_t(1) \geq \zeta_t(2) \leq \hdots \geq \zeta_t(K) \quad \text{for each $t=1,2,\dots,T$.}$$
We refer to $\{1,2,\dots,k_t\}$ the set of majority arms and $\{k_t+1,\dots,K\}$ the set of minority arms at round $t$.%
\footnote{We stress that in the $K$-arm setting, although $k_t$ is the minimum index such that $\zeta_t(1)+\cdots+\zeta_t(k_t) \geq \frac{1}{2}$, it may not be the minimum index so that $q_t(1)+\cdots+q_t(k_t) \geq \frac{1}{2}$.}
We let $M \defeq \sum_{t =1}^T \E \sum_{i \leq k_t} \bell_t(i)$ and $m \defeq \sum_{t =1}^T \E \sum_{i > k_t} \bell_t(i)$ respectively be the total loss of the majority and minority arms.
We again have
\begin{equation} \label{eq:omwtoshow2}
\textstyle \E \sum_{t=1}^T \|\bell_t\|_2^2 \leq \E \sum_{t=1}^T \sum_{i\in[K]} \bell_t(i) = M + m ~.
\end{equation}
Thus, the whole game to prove \eqref{eq:toshow} is to upper bound $M$ and $m$.

\subsection{Useful properties}

We state a few properties about $q_t$ and its truncations.
\begin{lemma}\label{claim:thres-alternative}
In each round $t=1,2,\dots,T$, if $q_t$ satisfies \eqref{eq:implicitdef}, then
$$
\text{for every $s\in S$ and $i\leq k_t \colon $\quad }
\xi^s_t(i) = \frac{\zeta_t(i)}{\sum_{j\leq k} \zeta_t(j)} \cdot \big( 1 - \sum_{j>k} \xi^s_t (j) \big)
$$
\end{lemma}
\begin{proof}
Let $i \leq k_t$ and $s \in S$. By \eqref{eq:implicittruncation} and since $\xi^s_t = \Trun_s^{k_t} q_t$ one has
\[
\xi_t^s(i) =\frac{q_t(i)}{\sum_{j \leq k} q_t(j)} \sum_{j \leq k} \xi_t^s(j) \,.
\]
Moreover $q_t$ is a mixture of $\zeta_t$ and truncated versions of $\zeta_t$ so similarly using \eqref{eq:implicittruncation} one has
\[
q_t(i) =\frac{\zeta_t(i)}{\sum_{j \leq k} \zeta_t(j)} \sum_{j \leq k} q_t(j) \,.
\]
Putting the two above displays together concludes the proof.
\end{proof}

\begin{lemma}\label{claim:maj-mino-property}
In each round $t=1,2,\dots,T$, if $q_t$ satisfies \eqref{eq:implicitdef}, then
\begin{itemize}
\item for every minority arm $i > k_t$ it satisfies $q_t(i) \leq \zeta_t(i)$,  and
\item for every majority arm $i\leq k_t$ it satisfies $q_t(i) \geq \zeta_t(i) \geq \frac{1}{2K}$.
\end{itemize}
\end{lemma}
\begin{proof}
For sake of notation we drop the index $t$ in this proof.
Recall $q = \sum_{e\in E\cup S} \frac{w(e)}{\|w\|_1} \cdot \xi^e$.
\begin{itemize}
\item For every minority arm $i> k$, every $s\in S$, we have $\xi^e(i) = \big(\Trun^k_s q\big) (i)\leq q(i)$ according to \rref{def:truncKdef}. Therefore, we must have $q(i) = \sum_{e\in E\cup S} \frac{w(e)}{\|w\|_1} \cdot \xi^e(i) \leq \frac{\sum_{e\in  E} w(e) \xi^e(i)}{\sum_{e\in  E} w(e)} = \zeta(i)$.
\item For every majority arm $i\leq k$, we have (using \rref{claim:thres-alternative})
$$\xi^e(i) = \frac{\zeta(i)}{\sum_{j\leq k} \zeta(j)} \cdot (1-\sum_{j> k} \xi^s(j)) \geq
\frac{\zeta(i)}{\sum_{j\leq k} \zeta(j)} \cdot (1-\sum_{j>k} \zeta(j)) = \zeta(i) $$
From the definition of $k = \min\{ i\in[K] \colon \sum_{j\leq i} \zeta(j) \geq \frac{1}{2} \}$, we can also conclude $\zeta(i) \geq \zeta(k) \geq \frac{1}{2K}$. This is because $\frac{1}{2} \leq \sum_{j>k} \zeta(j) \leq K \zeta(k)$.
\qedhere
\end{itemize}
\end{proof}

The next lemma shows that setting $p_t = \Trun^{k_t}_{\gamma} q_t$ satisfies the assumption of \rref{lem:1}.
\begin{lemma} \label{lem:ass1}
If $q_t$ satisfies \eqref{eq:implicitdef}, $\gamma\in(0,\frac{1}{2}]$ and $p_t = \Trun^{k_t}_{\gamma} q_t$, then for every arm $i\in [K]$:
$$(1- 2 K \gamma) p_t(i) \leq q_t(i) \quad\text{and} \quad p_t(i) \neq 0 \Rightarrow p_t(i) \geq q_t(i) \enspace.$$
\end{lemma}

\begin{proof}
For sake of notation we drop the index $t$ in this proof.

By \rref{def:truncKdef} and \rref{claim:maj-mino-property}, we have for every $i\in [K]$:
\begin{equation*}
p(i)
\leq q(i) \Big( 1 + \frac{\sum_{j:j>k \,\wedge\, q(j)\leq \gamma} q(j)}{\sum_{j\leq k} q(j)} \Big)
\leq q(i) \Big( 1 + \frac{\sum_{j:q(j)\leq \gamma} q(j)}{\sum_{j\leq k} \zeta(j)} \Big)
\leq q(i) (1+2K \gamma) \enspace.
\end{equation*}
The other statement follows because whenever $p(i)\neq 0$, \rref{def:truncKdef} says it must satisfy $p(i)\geq q(i)$.
\end{proof}

\subsection{Bounding $m$ and $M$}

We first upper bound $M$ and then upper bound $m$.
\begin{lemma} \label{lem:simplekeyK}
If $q_t$ satisfies \eqref{eq:implicitdef}, then $M \leq 2 K L_T$.
\end{lemma}

\begin{proof}
Using \rref{claim:maj-mino-property} we have $q_t(i) \geq \frac{1}{2K}$ for any $i \leq k_t$. 
Also, $p_t(i) \geq q_t(i)$ for every $i$ satisfying $\bell_t(i) > 0$ (owing to \rref{def:truncLoss} and \rref{lem:ass1}). Therefore,
\begin{align*}
M
= \sum_{t =1}^T \E \sum_{i \leq k_t} \bell_t(i)
&\leq 2 K \sum_{t =1}^T \E \sum_{i \leq k_t} q_t(i) \cdot \bell_t(i)
\leq 2 K \sum_{t =1}^T \E \sum_{i \leq k_t} p_t(i) \cdot \bell_t(i) \\
&\leq 2 K \sum_{t =1}^T \E \langle p_t, \bell_t \rangle
= 2 K \sum_{t =1}^T \E \langle p_t, \ell_t \rangle = 2 K L_T \enspace.
\tag*{\qedhere}
\end{align*}
\end{proof}

\begin{lemma} \label{lem:keyK}
Suppose $q_t$ satisfies \eqref{eq:implicitdef}, and denote by $L_t^s \defeq \E \sum_{t=1}^T \langle \Trun^{k_t}_s q_t, \ell_t \rangle = \E \sum_{t=1}^T \langle \xi^s_t, \ell_t \rangle$ the total expected loss of $q_t$ truncated to $s$.
Then, as long as $m - 2 K L_T > 0$,
\[
\exists s \in (\gamma,1/2] \cap \frac{1}{2T} \mathbb{N} \colon \quad m- 2 K L_T \leq \frac{1+L_T - L_T^s}{\gamma} ~.
\]
\end{lemma}

\begin{proof}
The proof is a careful generalization of the proof of \rref{lem:key2} (which in turn is just a discretization of the proof of \rref{lem:key}).
Recall the notation $\underline{x}$ for the smallest element in $[x,+\infty) \cap \frac{1}{2T} \mathbb{N}$, and observe that for $s \in \frac{1}{2T} \mathbb{N}$, $x \leq s \Leftrightarrow \underline{x} \leq s$.

\newcommand{\Lbb}{\ell^{\mathsf{maj}}}
Denote by
\[
\Lbb_t \defeq \sum_{i \leq k_t} \frac{q_t(i)}{\sum_{j \leq k_t} q_t(j)} \bell_t(i) \,.
\]
the weighted loss of the majority arms at round $t$. We have $\sum_{t=1}^T \Lbb_t \leq 2 L_T$ because $\sum_{j \leq k_t} q_t(j) \geq \sum_{j \leq k_t} \zeta_t(j) \geq \frac{1}{2}$ and $q_t(i)\leq p_t(i)$ whenever $\bell_t(i)>0$ (owing to \rref{def:truncLoss} and \rref{lem:ass1}).

Now, for any $s \geq \gamma$, define the function
$$\textstyle f(s) \defeq \E \sum_{t=1}^T \sum_{i > k_t} \ds1\{\underline{q_t(i)} \leq s\} (\Lbb_t - \bell_t(i)) \enspace.$$
Let us pick $s \in [\gamma,1/2] \cap \frac{1}{2T}\mathbb{N}$ to minimize $f(s)$, and breaking ties by choosing the smaller value of $s$.
We make several observations:
\begin{itemize}
\item $f(\gamma) \geq 0$ because for any $t$ and $i> k_t$ with $q_t(i) \leq \gamma$ we must have $p_t(i) = (\Trun^{k_t}_\gamma q_t)(i)=0$ and thus $\bell_t(i)=0$ by the definition of $\bell_t$ in \rref{def:truncLoss}.
\item $f(1/2) = \sum_{t=1}^T (K - k_t) \Lbb_t - m \leq 2 K L_T - m < 0$.
\item $s > \gamma$ because $f(s) \leq f(1/2) < 0$.
\end{itemize}

Let us define the points
\[
\text{$s_0 \defeq \gamma$ \quad and \quad $\{s_1 < \hdots < s_{m}\} \defeq (\gamma, s] \cap \bigcup_{i \in [K]} \{\underline{q_1(i)}, \hdots, \underline{q_T(i)}\}$.}
\]
Note that the tie-breaking rule for the choice of $s$ ensures $s_m = s$ (if $s_m<s$ then it must satisfy $f(s_m) = f(s)$ giving a contradiction).

Observe that by definition of the truncation operator, one has
\[
\langle \Trun^{k_t}_s q_t, \bell_t \rangle = \langle q_t, \bell_t \rangle + \sum_{i > k_t} \ds1\{q_t(i) \leq s\} q_t(i) (\Lbb_t - \bell_t(i))
\]
In fact, after rounding, one can rewrite the above for some $\epsilon_{s,t} \in [-\frac{1}{2T}, \frac{1}{2T}]$ as
\[
\langle \Trun^{k_t}_s q_t, \bell_t \rangle = \langle q_t, \bell_t \rangle + \epsilon_{s,t} + \sum_{i > k_t} \ds1\{\underline{q_t(i)} \leq s\} \underline{q_t(i)} (\Lbb_t - \bell_t(i))
\]Then, for some $\epsilon \in [-1,1]$, one has
\begin{align*}
L_T - L_T^s
&= \E \sum_{t=1}^T \langle \Trun^{k_t}_\gamma q_t - \Trun^{k_t}_s q_t, \ell_t \rangle
= \E \sum_{t=1}^T \langle \Trun^{k_t}_\gamma q_t - \Trun^{k_t}_s q_t, \bell_t \rangle \\
& = \epsilon + \E \sum_{t=1}^T \sum_{i > k_t} (\ds1\{\underline{q_t(i)} \leq \gamma\} - \ds1\{\underline{q_t(i)} \leq s\}) \underline{q_t(i)} (\Lbb_t - \bell_t(i)) \\
& = \epsilon + \E \sum_{j=1}^m \sum_{t=1}^T \sum_{i > k_t} - s_j \ds1\{\underline{q_t(i)} = s_j\} (\Lbb_t - \bell_t(i))\\
&=  \epsilon + \sum_{j =1}^m s_j (f(s_{j-1}) - f(s_{j})) = \epsilon + \sum_{j=1}^{m-1} (s_{j+1} - s_{j}) f(s_j) + s_1 f(s_0) - s_m f(s_m) ~.
\end{align*}

Since $f(s_0) = f(\gamma) \geq 0$, $f(s_i) \geq f(s)$ and $s = s_m$, we conclude that
\begin{equation*}
L_T - L_T^s \geq \epsilon + (s_m - s_1) f(s_m) - s_m f(s_m) = \epsilon - s_1 f(s_m) \geq \gamma (m - 2 K L_T) ~. \qedhere
\end{equation*}
\end{proof}

\subsection{Putting all together}
Finally, using \rref{lem:1} (which applies thanks to \rref{lem:ass1}), \eqref{eq:omwtoshow2} and $L_T^* \leq L_T^s$ (the loss of an expert is no better than the loss of the best expert $L_T^*$), we have
\begin{equation} \label{eq:toprove2}
\textstyle L_T - L_T^s \leq O\big( \frac{\log(|E'|)}{\eta} + \eta (m+M) + \gamma K L_T \big) ~.
\end{equation}
Putting this into \rref{lem:keyK} and then using $M \leq 2 K L_T$ from \rref{lem:simplekeyK},  we have for any $\gamma \geq 2 \eta$,
$$
\textstyle \gamma (m+M) \leq O\big( \frac{\log(|E'|)}{\eta} + \gamma K L_T \big) ~.
$$
Putting this into \eqref{eq:omwtoshow2}, we immediately get \eqref{eq:toshow} as desired.  This finishes the proof of \rref{th:main}. It only remains to ensure that $q_t$ verifying \eqref{eq:implicitdef} indeed exists. We provide an algorithm for this in \rref{sec:qt}.

\section{Algorithmic Process to Find $q_t$}
\label{sec:qt}

In this section, we answer the question of how to algorithmically find $q_t$ satisfying the implicitly definition \eqref{eq:implicitdef}. We recall \eqref{eq:implicitdef}:
\begin{equation*}
q_{t} = \frac{1}{\sum_{e \in E} w_t(e) + \sum_{s \in S} w_t(s)} \left(\sum_{e \in E} w_t(e) \xi_t^e + \sum_{s \in S} w_t(s) \Trun^{k_t}_s q_t \right) ~.
\tag*{\eqref{eq:implicitdef}}
\end{equation*}
We show the following general lemma:
\begin{lemma} \label{lem:itexists}
Given $k\in [K]$, a finite subset $S\subset \big[0, \frac{1}{2}\big]$, $\zeta \in \Delta_K$ with $\zeta(1)\geq \cdots \geq \zeta(K)$, and $W \in \Delta_{1+|S|}$, \algorithmref{alg:constructive} finds some $q \in \Delta_K$ such that
\[
q = W(1) \zeta + \sum_{s \in S} W(s) \Trun^k_s q  \,.
\]
Furthermore, \algorithmref{alg:constructive} runs in time $O(K\cdot |S|)$.
\end{lemma}
We observe that by setting $k=k_t$,
$$ 
\zeta = \zeta_t = \frac{\sum_{e\in E} w_t(e) \cdot \xi^e_t}{\sum_{e\in E} w_t(e)} \enspace,
\quad
W(1) = \frac{\sum_{e\in E} w_t(e)}{\|w_t\|_1}
\quad\text{and}\quad
\forall s\in S\colon W(s) = \frac{w_t(s)}{\|w_t\|_1}
$$
in \lemmaref{lem:itexists}, we immediately obtain a vector $q\in \Delta_K$ that we can use as $q_t$.

\parhead{Intuition for \rref{lem:itexists}}
We only search for $q$ that is non-increasing for minority arms. This implies $\Trun^k_s q$ is monotonically non-increasing for minority arms as well. In symbols:
$$
q(k+1)\geq \cdots \geq q(K)
\quad\text{and}\quad
(\Trun^k_s q)(k+1) \geq \cdots \geq (\Trun^k_s q)(K) \enspace.
$$
Due to such monotonicity, when computing $\Trun^k_s q$ for each $s\in S$, there must exist some index $\pi_s\in\{k+1,k+2,\dots,K+1\}$ such that the entry $q(i)$ gets zeroed out for all $i\geq \pi_s$
$$ \text{or in symbols, $(\Trun^k_s q)(i)=0$ for all $i\geq \pi_s$. } $$
Now, the main idea of \algorithmref{alg:constructive} is to search for such non-increasing function $\pi \colon S \to [K+1]$. It initializes itself with $\pi_s = k+1$ for all $s\in S$, and then tries  to increase $\pi$ coordinate by coordinate.

For each choice of $\pi$, \algorithmref{alg:constructive} computes a candidate distribution $q_\pi \in \Delta_K$ which satisfies 
\begin{equation}\label{eqn:q_pi-u}
q_\pi = W(1) \zeta + \sum_{s \in S} W(s) u_s
\end{equation}
where each $u_s$ is $q_\pi$ but truncated so that its probabilities after $\pi_s$ are redistributed to the first $k$ arms, or in symbols,
$$
u_s(i) = \left\{
                 \begin{array}{ll}
                   0, & \hbox{$i \geq \pi_s$;} \\
                   q_\pi(i), & \hbox{$\pi_s > i > k$;} \\
                   q_\pi(i) \cdot \big( 1 + \frac{\sum_{j:j\geq \pi_s} q_\pi(j)}{\sum_{j\leq k} q_\pi(j)} \big) , & \hbox{$i \leq k$.}
                 \end{array}
               \right.
$$
One can verify that the distribution $q_\pi \in\Delta_K$ defined in \lineref{line:q-pi} of \algorithmref{alg:constructive} is an explicit solution to \eqref{eqn:q_pi-u}. Unfortunately, each $u_s$ may not satisfy $\Trun^k_s q_\pi = u_s$. In particular, there may exist
$$ \text{some $s\in S$ and $i>k$ such that $q_\pi(i) > s$ but $u_s(i)=0$.} $$
This means, we may have truncated too much for expert $s$ in defining $u_s$, and we must increase $\pi_s$.

Perhaps not very surprisingly, if each iteration we only increase one $\pi_s$ by exactly $1$, then we never overshoot and there exists a moment when $q = q_\pi$ exactly satisfies 
$$q = W(1) \zeta + \sum_{s \in S} W(s) \Trun^k_s q  \,.$$

In the next subsection, we give a formal proof of \rref{lem:itexists}.
\begin{algorithm*}[t!]
\caption{ \label{alg:constructive}}
\begin{algorithmic}[1]
\Require $k\in [K]$, a finite set $S\subseteq \big[0, \frac{1}{2}\big]$, $\zeta \in \Delta_K$ with $\zeta(1)\geq \cdots \geq \zeta(K)$, and $W \in \Delta_{1+|S|}$
\vspace{2pt}
\Ensure $q \in \Delta_K$ such that $q = W(1) \zeta + \sum_{s \in S} W(s) \Trun^k_s q$.
\vspace{2pt}
\State initialize $\pi \colon S \to [K+1]$ as $\pi_s = k+1$;
\Comment{will ensure $\pi_s \in \{k+1,k+2,\dots,K+1\}$}
\While {true}
\State \label{line:q-pi} 
$q_\pi (i) \gets \left\{
                     \begin{array}{ll}
                       \frac{W(1)}{1-\sum_{s\in S \wedge \pi_s>i} W(s)} \cdot \zeta(i), & \hbox{if $i>k$;} \\
                       \frac{\zeta(i)}{\sum_{j\leq k} \zeta(j)} \cdot (1 - \sum_{j>k} q_\pi(j)), & \hbox{if $i\leq k$.}
                     \end{array}
                   \right.$
\Comment{$q_\pi \in \Delta_K$}

\State Pick any $s\in S$ with $\pi_s \leq K$ such that $q_\pi(\pi_s) > s$.
\label{line:pick-s}
\State
\textbf{if } $s$ is not found \textbf{ then break}

\State
\textbf{else }$\pi_s \gets \pi_s + 1$.
\EndWhile

\State \Return $q_\pi$.
\end{algorithmic}
\end{algorithm*}

\subsection{Proof details}
\begin{claim}\label{claim:constructive:property}
We claim a few basic properties about \algorithmref{alg:constructive}
\begin{enumerate}[label=(\alph{*}), ref=\ref*{claim:constructive:property}.\alph{*}]
\item \label{claim:constructive:property1} The process finishes after at most $K \cdot |S|$ iterations.
\item \label{claim:constructive:property2} We always have $q_\pi(k+1) \geq \cdots \geq q_\pi(K)$.
\item \label{claim:constructive:property3} As $\pi$ changes, for each minority arm $i > k$, $q_\pi(i)$ never decreases.
\item \label{claim:constructive:property4} When the while loop ends, for each $i>k$ and $s\in S$, we have $q_\pi(i)>s \Longleftrightarrow \pi_s > i$.
\end{enumerate}
\end{claim}
\begin{proof} $ $
\begin{enumerate}[label=(\alph{*})]
\item This is because each $\pi_s$ changes at most $K$ times.
\item This is because $\frac{W(1)}{1-\sum_{s\in S \wedge \pi_s>i} W(s)} \cdot \zeta(i) \geq \frac{W(1)}{1-\sum_{s\in S \wedge \pi_s>j} W(s)} \cdot \zeta(j)$ when $k<i<j$.
\item This is because $\frac{W(1)}{1-\sum_{s\in S \wedge \pi_s>i} W(s)} \cdot \zeta(i)$ never decreases as $\pi_s$ increases for each $s\in S$.
\item The proof of this statement relies on the previous ones. Recall when the while loop ends, for each $s$ we have either $\pi_s = K+1$ or $q_\pi(\pi_s) \leq s$ (recall \lineref{line:pick-s} of \algorithmref{alg:constructive}). Therefore, the monotonicity $q_\pi(k+1) \geq \cdots \geq q_\pi(K)$ from \claimref{claim:constructive:property2} tells us
$$ \pi_s \leq i \implies q_\pi(i) \leq q_\pi(\pi_s) \leq s \enspace. $$
On the other hand, if $\pi_s >i$, then denote by $\pi'$ be the most recent copy of $\pi$ where $\pi_s=i$. (Since each $\pi_s$ only increases there must exist such $\pi'$.) Now, in the immediate next iteration, $\pi_s'$ increases from $i$ to $i+1$, so we must have $q_{\pi'}(i) > s$ (recall \lineref{line:pick-s} of \algorithmref{alg:constructive}). Finally, since $q_\pi(i) \geq q_{\pi'}(i)$ due to \claimref{claim:constructive:property3}, we conclude that
\begin{equation*}
\pi_s > i \implies q_\pi(i) \geq q_{\pi'}(i) > s \enspace. \qedhere
\end{equation*}
\end{enumerate}
\end{proof}

\begin{proof}[Proof of \rref{lem:itexists}]
Suppose in the end of \algorithmref{alg:constructive} we obtain $q = q_\pi$ for some $\pi \colon S \to [K+1]$.
Let $\xi^s = \Trun^k_s q$ for each $s\in S$ and $q' = W(1) \zeta + \sum_{s \in S} W(s) \Trun^k_s q$. We need to show $q = q'$.
For every minority arm $i > k$:
\begin{align*}
q'(i)
&\overset\da= W(1) \cdot \zeta (i) + \sum_{s \in S} W(s) \cdot \xi^s(i)
\\
&\overset\db=
W(1) \cdot \zeta(i) + \Big(\sum_{s\in S \wedge q(i)>s} W(s) \Big) \cdot q(i)
\\
&\overset\dc=
W(1) \cdot \zeta(i) + \Big( \sum_{s\in S \wedge \pi_s > i} W(s) \Big) \cdot q(i)
\overset\dd= q(i) \enspace.
\end{align*}
Above, equality \da is by the definition of $q'$, equality \db is by the definition of $\xi^s = \Trun^k_s q$, equality \dc follows from \claimref{claim:constructive:property4}, and equality \dd is by definition of $q(i) = q_\pi(i) = \frac{W(1)}{1- \sum_{s\in S \wedge \pi_s > i} W(s)} \cdot \zeta(i)$.

For every majority arm $i\leq k$,
\begin{align}
\frac{q'(i)}{\zeta{(i)}}
&\overset\da=
W(1) \cdot \frac{\zeta(i)}{\zeta{(i)}} + \sum_{s\in S} W(s) \cdot \frac{\xi^s(i)}{\zeta{(i)}}
\nonumber \\
&\overset\db=
W(1)  + \sum_{s\in S} W(s) \cdot \frac{\sum_{j\leq k} \xi^s(j)}{\sum_{j\leq k} \zeta(j)} 
\label{eqn:q'}
\end{align}
where equality \da is by the definition of $q'$ and equality \db is because for every $i\leq k$ it satisfies $\frac{\xi^s(i)}{q(i)} = \frac{\sum_{j\leq k} \xi^s(j)}{\sum_{j\leq k}q(j)}$ (using definition of $\xi^s = \Trun^k_s q$) and for every $i\leq k$ it satisfies $\frac{\zeta(i)}{q(i)} = \frac{\sum_{j\leq k} \zeta(j)}{\sum_{j\leq k}q(j)}$ (using definition of $q = q_\pi$ \lineref{line:q-pi} of \algorithmref{alg:constructive}).

Now, the right hand side of \eqref{eqn:q'} is independent of $i$. Therefore, we can write $q'(i) = C_1 \cdot \zeta(i)$ for each $i \leq k$ with some constant $C_1> 0$. Our definition of $q = q_\pi$ (see \lineref{line:q-pi} of \algorithmref{alg:constructive}) ensures that we can also write $q(i) = C_2 \cdot \zeta(i)$ for each $i \leq k$ with some constant $C_2>0$. Therefore, since for every $i>k$ we have already shown $q'(i) = q(i)$, it must satisfy $C_1=C_2$ and therefore $q'(i) = q(i)$ for all $i\in [K]$.

After proving $q' = q$, we only need to argue about the running time. 

If \algorithmref{alg:constructive} is implemented naively, then the total running time is $O((K \cdot |S|)^2)$ because there are at most $K \cdot |S|$ iterations (see \claimref{claim:constructive:property1}) and in each iteration we can compute $q_\pi$ in time $O(K \cdot |S|)$. In fact it is rather easy to find implicit update rules to make each iteration of \algorithmref{alg:constructive} run in $O(1)$ time. We give some hints for this below.

Indeed, if in an iteration some $\pi_s$ is changed from $i$ to $i+1$ (recalling $i>k$), then we can update $q_\pi(i)$ in $O(1)$ time. For each $j>k$ where $j\neq i$, we have $q_\pi(j)$ is unchanged. The values of $q_\pi(j)$ for $j\leq k$ all need to be changed, but they are only changed altogether by the same multiplicative factor (which can again be calculated in $O(1)$ time). 

Finally, to search for $s\in S$ with $\pi_s \leq K$ and $q_\pi(\pi_s) > s$, we do not need to go through all $s\in S$. Instead, for each $i>k$, we maintain ``the smallest $s_i\in S$ so that $q_\pi(i)>s_i$.'' Then, whenever $\pi_{s_i}\leq i$, that means we can pick $s = s_i$ because $q_\pi(\pi_s) = q_\pi(\pi_{s_i}) \geq q_\pi(i) > s_i = s$. For such reason, one can maintain a first-in-first-out \texttt{list} to store all values of $i$ where $q_\pi(i)>s_i$. In each iteration of \algorithmref{alg:constructive} we simply pick the first element in \texttt{list} and perform the update. This changes exactly one $q_\pi(j)$ for $j>k$, and thus may additionally insert one element to \texttt{list}. Therefore, in each iteration we only need $O(1)$ time to find some $\pi_s$ to increase.
\end{proof}

\bibliographystyle{plainnat}
\bibliography{newbib}
\end{document}